\documentclass{article}

\pdfoutput=1
\usepackage{fullpage}
\usepackage[utf8]{inputenc}
\usepackage{graphicx}
\usepackage{amsthm}
\usepackage{amssymb}
\usepackage{amsmath}
\usepackage{amsfonts}
\usepackage{tabularx,colortbl}
\usepackage{multirow}
\usepackage{color}
\usepackage[pdftex]{hyperref}
\usepackage[english]{babel}
\usepackage{subfigure}
\usepackage{algorithm}
\usepackage{algorithmic}
\usepackage[sort&compress,numbers]{natbib}
\usepackage{authblk}

\definecolor{Orange}{rgb}{1,0.64,0}
\definecolor{lgray}{rgb}{0.85,0.85,0.85}

\newcommand{\argmin}{\operatornamewithlimits{arg\ min}}

\newtheorem{theorems}{theorems}
\newtheorem{thMinsod}[theorems]{Claim}

\makeindex

\begin{document}

\title{Data granulation by the principles of uncertainty}

\author[1]{Lorenzo Livi\thanks{llivi@scs.ryerson.ca}\thanks{Corresponding author}}
\author[1]{Alireza Sadeghian\thanks{asadeghi@ryerson.ca}}
\affil[1]{Dept. of Computer Science, Ryerson University, 350 Victoria Street, Toronto, ON M5B 2K3, Canada}
\renewcommand\Authands{, and }
\providecommand{\keywords}[1]{\textbf{\textit{Index terms---}} #1}

\maketitle

\begin{abstract}
Researches in granular modeling produced a variety of mathematical models, such as intervals, (higher-order) fuzzy sets, rough sets, and shadowed sets, which are all suitable to characterize the so-called information granules.
Modeling of the input data uncertainty is recognized as a crucial aspect in information granulation.
Moreover, the uncertainty is a well-studied concept in many mathematical settings, such as those of probability theory, fuzzy set theory, and possibility theory.
This fact suggests that an appropriate quantification of the uncertainty expressed by the information granule model could be used to define an invariant property, to be exploited in practical situations of information granulation.
In this perspective, we postulate that a procedure of information granulation is effective if the uncertainty conveyed by the synthesized information granule is in a monotonically increasing relation with the uncertainty of the input data.
In this paper, we present a data granulation framework that elaborates over the principles of uncertainty introduced by Klir.
Being the uncertainty a mesoscopic descriptor of systems and data, it is possible to apply such principles regardless of the input data type and the specific mathematical setting adopted for the information granules.
The proposed framework is conceived (i) to offer a guideline for the synthesis of information granules and (ii) to build a groundwork to compare and quantitatively judge over different data granulation procedures.
To provide a suitable case study, we introduce a new data granulation technique based on the minimum sum of distances, which is designed to generate type-2 fuzzy sets.
The automatic membership function elicitation is completely based on the dissimilarity values of the input data, which makes this approach widely applicable.
We analyze the procedure by performing different experiments on two distinct data types: feature vectors and labeled graphs. Results show that the uncertainty of the input data is suitably conveyed by the generated type-2 fuzzy set models.\\
\end{abstract}

\section{Introduction}

Granulation of information \cite{INT:INT20390,pedrycz2008handbook,6399462,Ulu20133713,liang2006information} emerges as an essential data analysis paradigm.
Information used or acquired to describe an abstract/physical/social process is usually expressed in terms of data (experimental evidence).
Therefore, granulation of information usually translates to data granulation.
Granulation of data can be roughly described as the action of aggregating semantically and functionally similar elements of the available experimental evidence.
This is performed to achieve a higher-level data description, which is implemented in terms of information granules (IGs) \cite{pedrycz2013granular}.
IGs are sound data aggregates that are formally described by a suitable mathematical model. Many mathematical settings have been proposed so far in the related literature, such as intervals--hyperboxes, (higher order) fuzzy sets, rough sets, and shadowed sets \cite{pedrycz2008handbook}.
The synthesized IGs can be used for interpretability purposes \cite{Mencar20084585,mencar_interpretability} or they can be used as a computational component of a suitable intelligent system \cite{apolloni2008interpolating,pedrycz2013granular,gralg_2012,pedrycz2011granular,si_asoc_grc,ganivada2011fuzzy,Pedrycz20083720,zhang2008granular,melin2013review,lu2014modeling}.
In any case, the problem of designing effective and justifiable data granulation procedures (GPs) is of paramount importance \cite{Choi20092102,6247495,usit2_2012,livi_hooman_it2fs_fitk,Castillo20115590,pedrycz2008dynamic,huang2013information,song2014human}.

The principle of justifiable granularity (PJG) is a well-established guideline for the synthesis of IGs \cite{6459568,Pedrycz20134209,pedrycz_timeseries}.
The PJG states that granulation should be performed by finding the ``optimal'' compromise among two conflicting requirements: specificity and generality. In other terms, an IG modeling input data should be designed such that it retains only the essential information (it should be specific, conveying a specific semantic content) but, at the same time, it should cover a reasonable amount of information.
Since the PJG is conceived to provide an adaptive mechanism to the information granulation problem, it is not designed to directly offer a built-in mechanism to objectively evaluate the quality of the granulation itself.
To this end, it is necessary to rely on external performance measures to quantify and judge over the quality of an IG.

The uncertainty is a peculiar property of virtually every human action that involves reasoning, decision making, and perception \cite{utkin2007decision,yager1992decision}.
Modeling the uncertainty of the input data is an essential mission in data granulation.
In fact, any IG model is designed to handle and hence express the uncertainty through an appropriate formalism. How the uncertainty is embedded into the IG model depends, of course, on the specific mathematical setting used for the IG.
However, while the numerical quantification of the uncertainty pertaining a specific situation may change as we change the mathematical setting of the IG, the \textit{level} of uncertainty should remain the same.
In these terms, the principles of uncertainty \cite{Klir199515,uncertainty_klir_2001} offer a compelling guideline to implement and evaluate practical data granulation techniques.

In this paper, we elaborate a conceptual data granulation framework over the principles of uncertainty.
A preliminary version of the herein exposed ideas appeared in \cite{entropy_preserving_graph_it2}. Here we further elaborate over this preliminary work by providing a more extensive discussion of the framework, offering new experiments that demonstrate the different facets underlying such ideas.
In the proposed framework we idealize the uncertainty as an ``invariant'' property, to be preserved as much as possible during the granulation of the input data.
As a consequence, we are able to objectively quantify the effectiveness of the granulation, regardless of the input data representation and the adopted IG model.
This interpretation allows also to quantitatively judge on a common groundwork different data granulation techniques operating on the same data.
We provide a demonstration of these ideas by discussing a data granulation technique that generates type-2 fuzzy sets (T2FSs).

This is paper is structured as follows. Sec. \ref{sec:poui} introduces the principles of uncertainty.
Throughout Sec. \ref{sec:granulation} we introduce the proposed conceptual framework for data granulation.
In Sec. \ref{sec:t2-minsod} we present a procedure to generate T2FSs by means of the minimum sum of distances (MinSOD) technique.
In Sec. \ref{sec:experiments} we discuss the experiments and related results. Sec. \ref{sec:conclusions} concludes the paper.
We provide two appendices: \ref{sec:t2fs} introduces to the context of T2FSs, while \ref{sec:minsod} the MinSOD.

\section{The Principles of Uncertainty}
\label{sec:poui}

The principles of uncertainty have been introduced by \cite{Klir199515} two decades ago, with the aim of providing high-level guidelines to the development of well-justified methods for problem solving in presence of uncertainty.
Such principles elaborate over the ubiquitous concepts of uncertainty and information.
It is intuitive to understand that uncertainty and information are intimately related: the reduction of uncertainty is caused by gaining new information, and vice versa.

Three principles have been introduced (quotes are taken from \cite{Klir199515}):
\begin{enumerate}
 \item Principle of minimum uncertainty: ``It facilitates the selection of meaningful alternatives from solution sets obtained by solving problems in which some of the initial information is inevitably reduced in the solutions to various degrees. By this principle, we should accept only those solutions in a given solution set for which the information reduction is as small as possible.'';
 \item Principle of maximum uncertainty: ``This is reasoning in which conclusions are not entailed in the given premises. Using common sense, the principle may be expressed by the following requirement: in any ampliative inference, use all information available but make sure that no additional information is unwittingly added.'';
 \item Principle of uncertainty invariance: ``The principle requires that the amount of uncertainty (and information) be preserved when a representation of uncertainty in one mathematical theory is transformed into its couterpart in another theory.''.
\end{enumerate}

A combination of the first and third principle provides a compelling guideline for the purpose of data granulation. In fact, granulation of information implies mapping some input data (experimental evidence) originating from a certain input domain, say $\mathcal{X}$, to a domain of IGs, say $\mathcal{Y}$.
We argue that, when performing such a mapping, the uncertainty, regardless of the adopted formal mathematical framework, should be considered as an invariant property to be preserved as much as possible.

\section{Data Granulation with the Principles of Uncertainty}
\label{sec:granulation}

In this section, we introduce the proposed data granulation framework.
Fig. \ref{fig:mapping} illustrates the data granulation process. A procedure of data granulation can be formalized as a mapping, $\phi(\cdot)$, among two domains: input domain, $\mathcal{X}$, and the output domain, $\mathcal{Y}$.
$\mathcal{X}$ is the domain of the input data, whereas $\mathcal{Y}$ is a domain of IGs (e.g., a domain of hyperboxes, fuzzy sets, shadowed sets, rough sets and so on).
In practice, $\phi(\cdot)$ is a formal procedure for mapping a finite input dataset $\mathcal{S}\in\mathcal{P}_{<\infty}(\mathcal{X})$ with an output IG, say $\widetilde{\mathcal{A}}\in\mathcal{Y}$, i.e., $\widetilde{\mathcal{A}}=\phi(\mathcal{S})$.
Please note that we used a special mapping, $\mathcal{P}_{<\infty}(\cdot)$, in the input domain to allow discussing about $\mathcal{S}$ in terms of ``element'' of the input domain; in the following $\mathcal{P}_{<\infty}(\mathcal{X})$ is assumed to return all $n$-subsets of $\mathcal{X}$, with $n$ finite.
Note that $\mathcal{Y}$, as well as $\widetilde{\mathcal{A}}$, should be denoted by making explicit reference to $\mathcal{X}$ and $\mathcal{S}$, respectively, since IGs depend on the input. However, if no confusion is possible, we will avoid such specifications.
\begin{figure}[ht!]
 \centering
 \includegraphics[viewport=0 0 449 123,scale=0.43,keepaspectratio=true]{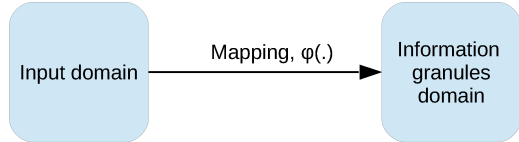}
 \caption{Data granulation as a mapping, $\phi: \mathcal{P}_{<\infty}(\mathcal{X})\rightarrow\mathcal{Y}$.}
 \label{fig:mapping}
\end{figure}

There are a number of important questions that should be answered: ``Is the mapping $\phi(\cdot)$ well-justified? Moreover, how do we asses objectively the quality of the mapping?''; ``Are there invariant properties that must be preserved in the transformation from $\mathcal{S}$ to $\widetilde{\mathcal{A}}$?''; ``Can we numerically quantify those properties?''; ``Given two GPs, are we able to affirm that one performs a better granulation than the other by considering the same experimental conditions?''.
Reasoning over those questions provides important motivations for the design and formal evaluation of information GPs.

IGs are semantically sound constructs that are synthesized to convey higher-level information with respect to (w.r.t) the data from which they are generated \cite{pedrycz2013granular}.
All models used in information granulation \cite{pedrycz2008handbook} are designed to realize a ``simplification'' of the input data. The simplification consists in aggregating data that are considered indistinguishable (also termed indiscernible) and functionally/semantically related. IGs are hence designed also to handle the uncertainty caused by this simplification.
How the uncertainty is handled by the IG model depends on the specific mathematical setting used to describe the IG \cite{uncertainty_klir_2001}.
However, it is a reasonable assumption that, regardless of the specific mathematical setting, two IGs with different models, but synthesized from the same input data, should convey a comparable uncertainty, i.e., they should agree at least on the ``level of uncertainty''. The same concept holds for the uncertainty measured in the input with the one measured in the resulting output IG.

In the following, we formalize a conceptual framework to design and evaluate specific implementations of the mapping $\phi(\cdot)$. We refer to the proposed framework as the Principle of Uncertainty Level Preservation (PULP).
Usually, $\mathcal{X}$ is a domain of non-granulated data, such as $\mathbb{R}^d$ vectors, sequences of objects, or graphs. However, $\mathcal{X}$ can be conceived also as a domain of IGs. In this case, since the role of $\phi(\cdot)$ is to provide an abstraction, $\mathcal{Y}$ must be a domain of higher-level IGs w.r.t. those of $\mathcal{X}$.
In the following, however, we will consider mappings from input domains of non-granulated data types only.

\subsection{Minimization of the Input--Output Uncertainty Difference}
\label{sec:input-output_min}

First and most important component of PULP is a measure to calculate the uncertainty.
Let $\hat{H}: \mathcal{P}_{<\infty}(\mathcal{X})\rightarrow\mathbb{R}^+$ and $\check{H}: \mathcal{Y}\rightarrow\mathbb{R}^+$ be, respectively, the uncertainty measures for the input and output domain.
Experimental evidence is usually collected in the form of a finite dataset, $\mathcal{S}$, containing $n=|\mathcal{S}|$ patterns/samples proper of the input domain, $\mathcal{X}$. To provide a better description of the properties of the data in $\mathcal{S}$, usually it is idealized an underlying data generating process, $P$, which actually generates instances of $\mathcal{S}$.
This abstract process can be characterized by a deterministic analytical model known in closed-form, a non-deterministic model that is assumed to provide a suitable description of $P$, or an unknown model. In the last case, which is the most common one, the only useful information that is available is the finite dataset, i.e., $\mathcal{S}$.
For all practical purposes, the dataset $\mathcal{S}$ is usually assumed to be representative of all important statistics of the underlying process $P$.
Therefore, in the following we center our discussion on $\mathcal{S}$.

Let $\hat{H}(\mathcal{S})$ and $\check{H}(\phi(\mathcal{S}))$ be, respectively, the uncertainty calculated for the input dataset and the synthesized IG.
According to the guidelines conceptualized in PULP, we define the granulation error (GE) as
\begin{equation}
\label{eq:ge}
\delta = \lVert\hat{H}(\mathcal{S}) - \psi(\check{H}(\phi(\mathcal{S})))\rVert,
\end{equation}
where $\lVert\cdot\rVert$ is a norm and $\psi: \mathbb{R}^+\rightarrow\mathbb{R}^+$ is a monotonically increasing function that it is used to map the different formalizations of the uncertainty (see Fig. \ref{fig:mapping_uncertainty} for an illustration). $\psi(\cdot)$ must be monotone increasing since, reasonably, if the input uncertainty increases, then the output uncertainty must increase as well, although the increments could be of a different extent.
Eq. \ref{eq:ge} provides a formal way to affirm that $\phi(\cdot)$ is characterized by a GE equal to $\delta$, which is important also to compare different procedures operating on the same data. Clearly, the lower the error the better the procedure.
\begin{figure}[ht!]
 \centering
 \includegraphics[viewport=0 0 449 123,scale=0.43,keepaspectratio=true]{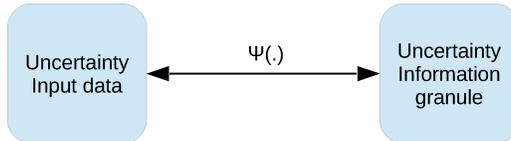}
 \caption{Function $\psi(\cdot)$ that provides a bridge among the input and output quantifications of the uncertainty.}
 \label{fig:mapping_uncertainty}
\end{figure}

From the theoretical viewpoint, the best GP, given $\mathcal{S}$, is the one that minimizes the GE,
\begin{equation}
\label{eq:opt_ig_procedure}
\phi(\cdot)^* = \argmin_{\phi(\cdot)\in\Phi} \lVert\hat{H}(\mathcal{S}) - \psi(\check{H}(\phi(\mathcal{S})))\rVert^2,
\end{equation}
where $\Phi$ is the set of all GPs suitable for the task at hand.
However, closed-form expressions for either $\phi(\cdot)$ and $\psi(\cdot)$ are necessary to evaluate either (\ref{eq:ge}) and (\ref{eq:opt_ig_procedure}). Moreover, a definition for the search space, $\Phi$, is also necessary in the case of Eq. \ref{eq:opt_ig_procedure}.

In practice, GPs are usually implemented as algorithms that most of times depend on some parameters, say $p$.
This is explicitly formalized by writing $\phi(\cdot; p)$. Therefore, to evaluate the quality of the mapping provided by a specific GP, we propose to deal with following optimization problem:
\begin{equation}
\label{eq:opt_params}
\delta^* = \min_{p} \lVert \hat{H}(\mathcal{S}) - \psi(\check{H}(\phi(\mathcal{S}; p))) \rVert^2.
\end{equation}

The optimal solution to (\ref{eq:opt_params}) yields the minimum GE, $\delta^*$, achievable for $\phi(\cdot)$.
Also $\delta^*$ can be used to objectively judge over the quality of the granulation provided by $\phi(\cdot)$.

The definition of the function $\psi(\cdot)$ is an important problem to be addressed. Such a function plays the role of the transformation among the different formalisms used to handle the uncertainty in the input (data) and output (granule) domains.
Defining $\psi(\cdot)$ in closed-form might be difficult, although it could be possible in specific cases (see Sec. \ref{sec:anal} for an example).
If this is the case, and also $\phi(\cdot)$ is available in closed-form, Eq. \ref{eq:opt_params} can be solved directly.
Optimization strategies to deal with (\ref{eq:opt_params}) depend on many factors, such as the nature of the uncertainty measures ($\hat{H}(\cdot)$ and $\check{H(\cdot)}$), and most importantly the specific implementation of $\phi(\cdot)$.
Providing general guidelines to the design of the specific optimization strategy is beyond the scope of this paper.

A universally valid method to obtain $\psi(\cdot)$ from a given problem instance, is via a suitable best-fitting algorithm. This can be performed, for instance, by analyzing $m$ i.i.d. dataset instances, $\mathcal{S}_i, i=1, 2, ..., m$, sampled from the same underlying data generating process.
The approximation of $\psi(\cdot)$, say $f(\cdot)$, is hence determined by fitting $m$ pairs of numbers obtained by the respective evaluations of $\hat{H}(\mathcal{S})$ and $\check{H}(\phi(\mathcal{S}; p))$.
The approximation would be characterized by a fitting error, $\epsilon\geq0$, which depends on $m$ and on the non-linearity of the underlying relation among the formalizations of the input/output uncertainty.
$\epsilon$ can be used in place of $\delta^*$ when the analytical definition of $\psi(\cdot)$ is not available.
Notably, the GE induced by the best-fitting error is defined as
\begin{equation}
\label{eq:ge_epsilon}
\epsilon = \lVert \hat{H}(\mathcal{S}) - f^*(\check{H}(\phi(\mathcal{S}))) \rVert,
\end{equation}
where $f^*(\cdot)$ is the optimal best-fitting function, which is derived by searching for the parameters $p$ of $\phi(\cdot)$ that generate $m$ samplings yielding the minimum fitting error.

\subsection{Brief Qualitative Discussion on the Proposed Framework}

The main contribution of PULP is a built-in formal criterion that is exploitable to judge over results of data granulation.
This is possible via the analysis of a quantity, GE, which is defined as the difference among the uncertainty measured in the input with the one calculated in the resulting output IG. In fact, the uncertainty in PULP is intended as an invariant property, to be preserved as much as possible during the granulation.

This particular aspect, which offers also a diagnostic tool of practical importance, is not included in the PJG. In fact, the PJG provides a guideline to design GPs by considering the two conflicting requirements of specificity (essentiality) and coverage -- an IG should be synthesized by finding a problem-dependent compromise among suitable implementations of those two factors.
Judging over the quality of the granulation is thus possible only indirectly, via human interpretation, by calculating some quality index, or by considering the synthesized IG as the input/component of another system (e.g., by using the IG in a classification system, evaluating thus the quality of the granulation as the accuracy of the classification).

If the input data contain outliers, a procedure designed by following the guidelines offered PULP would reflect also the contribution provided by those specific ``degenerate'' patterns; if not previously removed. In fact, the mapping is evaluated according to the capability of preserving the input uncertainty in the output IG model.
On the other hand, a procedure implemented according to the PJG, would determine an essential subset of input patterns, synthesizing an IG without considering those degenerate data.

Determining which approach offers a better solution, however, is something that depends on the context of application and on the ultimate needs of the user -- PULP is not proposed as a ``replacement'' for the PJG.

\section{A Type-2 Fuzzy Set Membership Functions Elicitation Method based on the MinSOD}
\label{sec:t2-minsod}

In this section we present a practical implementation of the mapping $\phi(\cdot)$ that is based on the MinSOD (see Sec. \ref{sec:minsod}).
The idea is to equip the MinSOD with the capability of generating an IG $\widetilde{\mathcal{A}}$ from $\mathcal{S}$, modeled as a T2FS (an introduction is provided in \ref{sec:t2fs}).
In the following, we refer to such a MinSOD as T2-MinSOD. Algorithm \ref{alg:t2minsod_algorithm} delivers the relative pseudo-code.
Given a dataset $\mathcal{S}$ and a dissimilarity measure $d:\mathcal{X}\times\mathcal{X}\rightarrow\mathbb{R}^+$, the T2-MinSOD can be described as a pair $(\nu, \mu_{\widetilde{\mathcal{A}}}(\cdot))$, where $\nu\in\mathcal{S}$ is the MinSOD representative and $\mu_{\widetilde{\mathcal{A}}}(\cdot)$ is the fuzzy membership function characterizing $\widetilde{\mathcal{A}}$.
Focusing on the IT2FS case, the membership degree of each $x\in\mathcal{S}$ is an interval $\mu_{\widetilde{\mathcal{A}}}(x)=[\mathrm{LMF}_{\widetilde{\mathcal{A}}}(x), \mathrm{UMF}_{\widetilde{\mathcal{A}}}(x)]$ in $[0, 1]$, bounded by LMF and UMF.

The element $\nu$ can be understood a suitable representative of $\mathcal{S}$ (a prototype). In fact, when $\mathcal{X}=\mathbb{R}$, $\nu$ corresponds to the median of the input dataset (see Claim \ref{th:Minsod}).
We exploit this fact in a more general setting, that is, when $\mathcal{X}$ is a user-defined data domain.
This allows us to interpret $\nu$, regardless of the nature of $\mathcal{X}$, as a well-justified representative of $\mathcal{S}$.
Accordingly, the evaluation of the UMF at $\nu$ should be one, denoting full membership.
Analogously, UMF for all other elements in $\mathcal{S}\setminus\{\nu\}$ should be defined by considering the dissimilarity value w.r.t. $\nu$.
This choice is motivated by the fact that we are trying to represent the uncertainty of $\mathcal{S}$ by relying only on the dissimilarity values among its elements.
The UMF is hence defined as a function of the dissimilarity w.r.t. $\nu$:
\begin{equation}
\label{eq:umf}
\mathrm{UMF}_{\widetilde{\mathcal{A}}}(x)=u(d(x, \nu)).
\end{equation}

$u(\cdot)$ is a monotonically non-increasing function of the argument yielding values in $[0, 1]$.
Candidate functions for (\ref{eq:umf}) are the Gaussian, rational quadratic kernel, or a linear functional in $[0, 1]$, like $M-d(x, y)$, where $M$ is the maximum value assumed by $d(\cdot, \cdot)$.

To form an interval membership function, we need to generate also the LMF.
The interval width in an IT2FS quantifies the uncertainty in describing the membership degree of an input element: the wider the interval, the higher the uncertainty.
LMF is determined as follows,
\begin{equation}
\label{eq:lmf}
\mathrm{LMF}_{\widetilde{\mathcal{A}}}(x)=\max\{\mathrm{UMF}(x) - l(D(x, \cdot)), 0\}.
\end{equation}

LMF is formed by considering the difference among the UMF and a function, $l(\cdot)$, of the dissimilarity w.r.t. the whole dataset -- note that $D(x, \cdot)$ denotes the set of dissimilarity values of all elements in $\mathcal{S}$ w.r.t. $x$.
$l(\cdot)$ could be implemented such that to capture the extent of the intra-granule dissimilarity values distribution (e.g., via the average or standard deviation etc.).
In this way, the uncertainty expressed by T2-MinSOD increases along with the diversity of the input patterns.
\begin{algorithm}[h!]\scriptsize
\caption{T2-MinSOD algorithm.}
\label{alg:t2minsod_algorithm}
\begin{algorithmic}[1]
\REQUIRE Dataset $\mathcal{S}$, a dissimilarity measure $d(\cdot, \cdot)$, functions $u(\cdot)$ and $l(\cdot)$
\ENSURE A T2-MinSOD, $(\nu, \mu_{\widetilde{\mathcal{A}}}(\cdot))$
\STATE According to Eq. \ref{eq:minsod}, determine $\nu$ on $\mathcal{S}$ using $d(\cdot, \cdot)$
\STATE Generate the interval-membership function, $\mu_{\widetilde{\mathcal{A}}}(\cdot)$, in the following loop
\FOR{\textbf{each} $x$ in $\mathcal{S}$}
\STATE Compute $\mathrm{UMF}_{\widetilde{\mathcal{A}}}(x)$ using $u(\cdot)$ with Eq. \ref{eq:umf}
\STATE Compute $\mathrm{LMF}_{\widetilde{\mathcal{A}}}(x)$ using $l(\cdot)$ with Eq. \ref{eq:lmf}
\ENDFOR
\RETURN $(\nu, \mu_{\widetilde{\mathcal{A}}}(\cdot))$
\end{algorithmic}
\end{algorithm}

T2-MinSOD could be exploited in many practical ways. For example, by using it (i) as an IG, thus providing a solution to analyze (and interpret) the uncertainty of $\mathcal{S}$, and (ii) as a computational component of a suitable intelligent system, which operates through data aggregation (e.g., a clustering-based procedure).
The T2-MinSOD does not just construct an IT2FS membership function, since in fact it allows also to easily defuzzify the granule by considering the representative $\nu\in\mathcal{S}$.

\section{Experiments}
\label{sec:experiments}

In Sec. \ref{sec:anal}, we discuss an example in which the input--output uncertainty mapping is solvable analytically.
Successively, we provide experiments considering T2-MinSOD operating in two different input domains: (i) Euclidean space and (ii) a domain of labeled graphs.
The first experiment (\ref{sec:exp_mv}) provides us also the possibility to visualize the results.
In Sec. \ref{sec:in-out_pres} we demonstrate that the T2-MinSOD is capable of generating IT2FS models that preserve the input uncertainty with reasonable GEs.
The second experiment (\ref{sec:exp_lg}) is performed considering several datasets of labeled graphs.
Finally, we discuss an experiment where T2-MinSOD is used in the clustering context (\ref{sec:clustering}).

Results are presented by implementing $u(\cdot)$ in Eq. \ref{eq:umf} as a Gaussian kernel -- dependent on the width -- and $l(\cdot)$ of Eq. \ref{eq:lmf} as the average.
We generate IGs modeled as IT2FSs. We rely on uncertainty measures based on entropy (see Refs. \cite{Wu:2009:CSR:1502817.1503036,Wu_UncIT2_2007} for detailed discussions on related measures of uncertainty for IT2FSs).
In particular, the uncertainty of the generated IT2FSs is computed by evaluating the (normalized) fuzzy entropy formulation given by \cite{Burillo1996305},
\begin{equation}
\label{eq:it2fs_entropy}
\check{H}(\widetilde{\mathcal{A}}) = \left(\frac{1}{|\mathcal{S}|}\sum_{i=1}^{|\mathcal{S}|} \overline{\mu}_{\widetilde{\mathcal{A}}}(x_i) - \underline{\mu}_{\widetilde{\mathcal{A}}}(x_i)\right) \in [0, 1].
\end{equation}

Analogously, we will characterize the uncertainty of the input by a suitable entropy measure.

\subsection{A Problem Solvable Analytically}
\label{sec:anal}

Let $\mathcal{S}=\{x_1, x_2, ..., x_n\}$ be a dataset of $n$ elements sampled from a Gaussian data generating process (also called source).
In this case, we calculate the (Shannon) entropy in closed-form \cite{cover2006elements}:
\begin{equation}
\label{eq:Gaussian_source_entropy}
\hat{H}(\mathcal{S}) = 1/2\ln(2\pi e \sigma^2),
\end{equation}
where $\sigma$ is the variance that completely characterizes the source (the higher the variance, the higher the entropy). For the purpose of this example, let us assume $\sigma\in[0, 1]$. Note that Eq. \ref{eq:Gaussian_source_entropy} actually holds for a dataset $\mathcal{S}$ as $n\rightarrow\infty$.
Now, let us define $\phi(\cdot)$ as a mapping that takes $\mathcal{S}$ and generates an IT2FS, $\widetilde{\mathcal{A}}$, with $\mu_{\widetilde{\mathcal{A}}}(x_i)=[0, \sigma], \forall x_i\in\mathcal{S}$.
Therefore, by evaluating (\ref{eq:it2fs_entropy}) on $\widetilde{\mathcal{A}}$, we have $\check{H}(\phi(\mathcal{S}))=\sigma$.
The definition of the function $\psi(\cdot)$ to map the uncertainty is straightforward.
In fact, $\psi(\check{H}(\phi(\mathcal{S})))=\psi(\sigma)=1/2\ln(2\pi e \sigma^2)=\hat{H}(\mathcal{S})$.
As a consequence, we can evaluate Eq. \ref{eq:ge} directly, obtaining $\delta= 0$, as $n\rightarrow\infty$.
It is worth noting that, according to Eq. \ref{eq:opt_ig_procedure}, such a GP $\phi(\cdot)$ is optimal.

\subsection{Tests on Real-valued Vectors}
\label{sec:exp_mv}

To grasp the concept in a data-driven scenario, in Fig. \ref{fig:1dvec} we show a sample of 100 patterns distributed according to a 1-dimensional Gaussian distribution with zero mean and $\sigma=0.2$.
In Fig. \ref{fig:1dvec_it2_w} we show the width (as a red spike) of the interval membership calculated by T2-MinSOD that characterizes each input pattern. As desired, patterns closer to the center of the distribution have a shorter interval width.
In Fig. \ref{fig:1dvec_it2} we show a representation of the generated IT2FS. It is possible to clearly recognize the Gaussian shape for both LMF and UMF. Notably, the interval-valued memberships are distributed as expected: the inner parts closer to the center are characterized by a ticker interval that, considering both endpoints, it is also closer to one.
\begin{figure*}[ht!]
\centering

\subfigure[Interval membership values generated for the input 1-dimensional patterns.]{
 \includegraphics[viewport=0 0 346 247,scale=0.55,keepaspectratio=true]{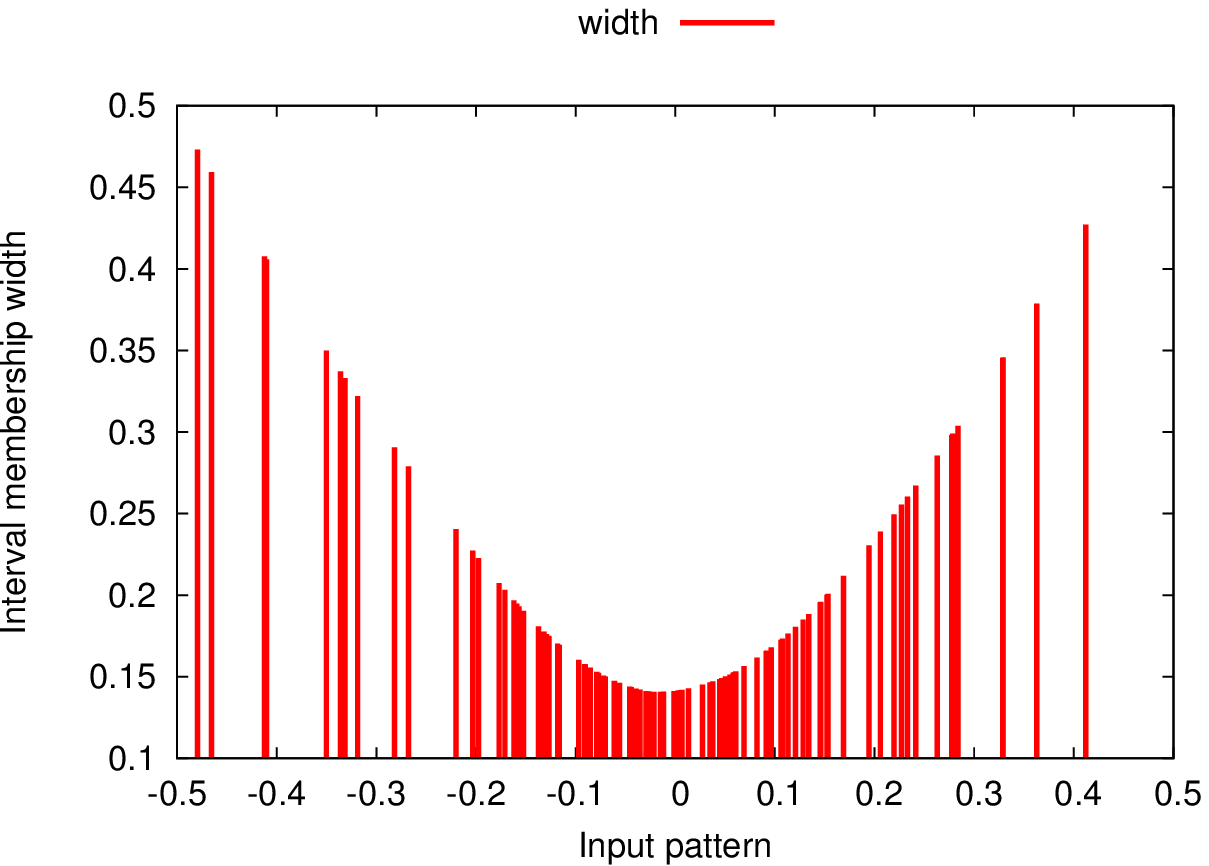}
\label{fig:1dvec_it2_w}}
~
\subfigure[Generated IT2FS.]{
 \includegraphics[viewport=0 0 346 247,scale=0.55,keepaspectratio=true]{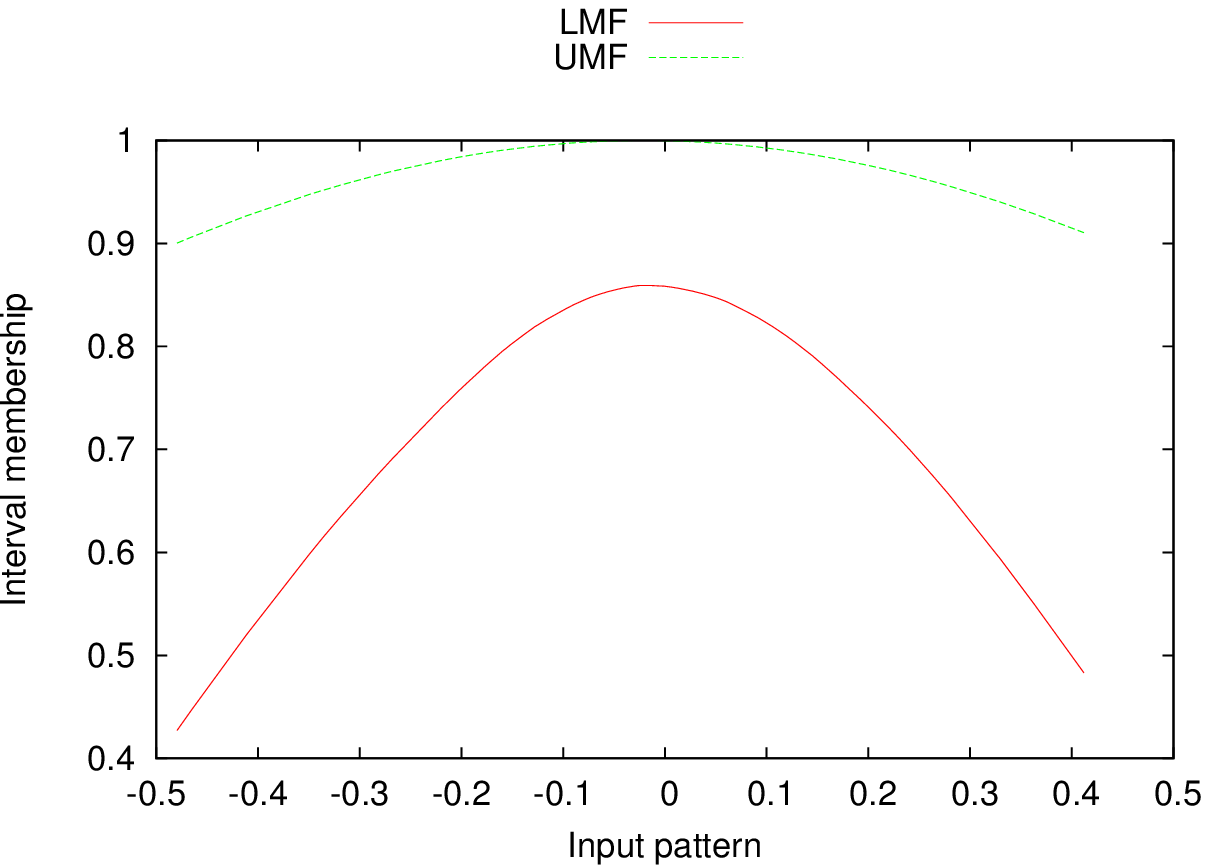}
\label{fig:1dvec_it2}}

\caption{Widths of the interval memberships assigned to the input 1-dimensional patterns (\ref{fig:1dvec_it2_w}) and a representation of the generated IT2FS (\ref{fig:1dvec_it2}).}
\label{fig:1dvec}
\end{figure*}

Fig. \ref{fig:2dvec} shows a dataset distributed according to a 2-dimensional Gaussian distribution, with zero mean and spherical covariance matrix controlled by $\sigma=0.2$. Fig. \ref{fig:2dvec_distr} shows the obtained configuration of the 100 sampled patterns -- the green pattern is the computed MinSOD element. Fig. \ref{fig:2dvec_it2} shows the LMF and UMF of the generated IT2FS. It is worth noting that, in this case, input patterns have no trivial ordering, which justifies the visualization (\ref{fig:2dvec_it2}) of the generated IT2FS.
\begin{figure*}[ht!]
\centering

\subfigure[Distribution of the 2-dimensional input vectors.]{
 \includegraphics[viewport=0 0 334 238,scale=0.55,keepaspectratio=true]{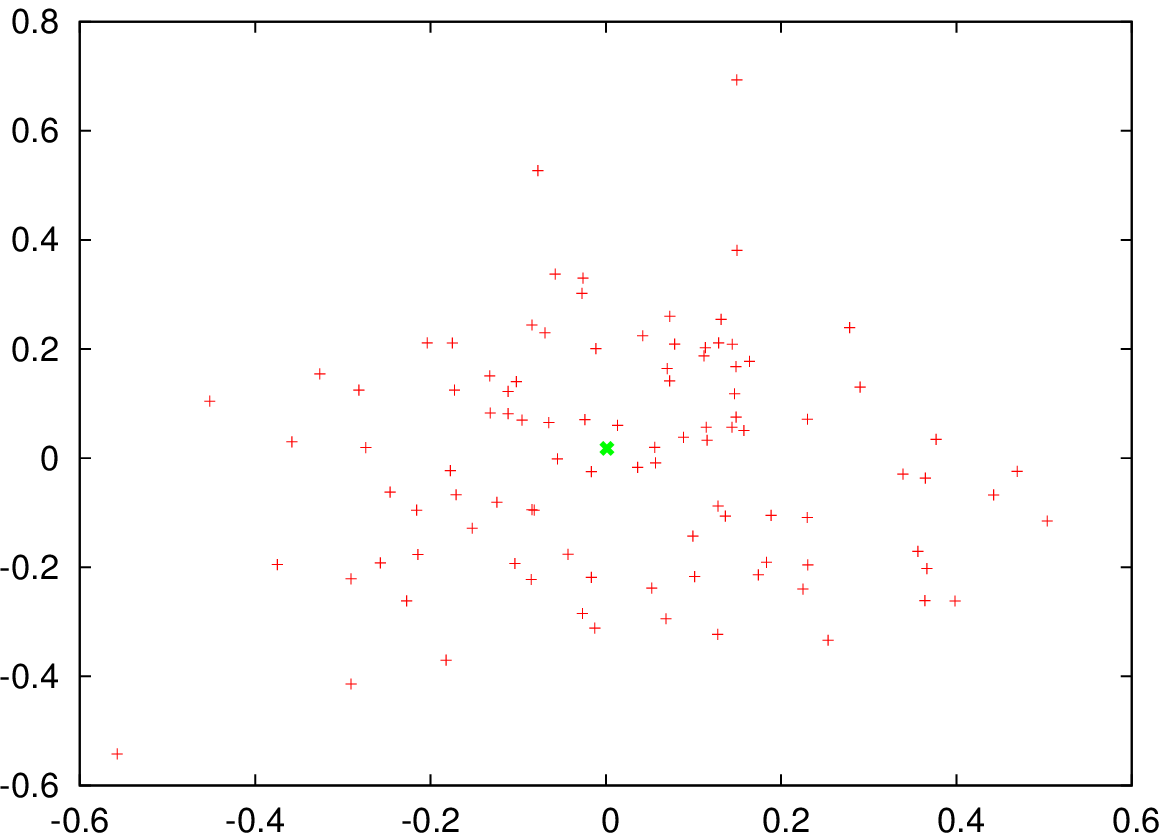}
\label{fig:2dvec_distr}}
~
\subfigure[Generated IT2FS.]{
 \includegraphics[viewport=0 0 348 247,scale=0.55,keepaspectratio=true]{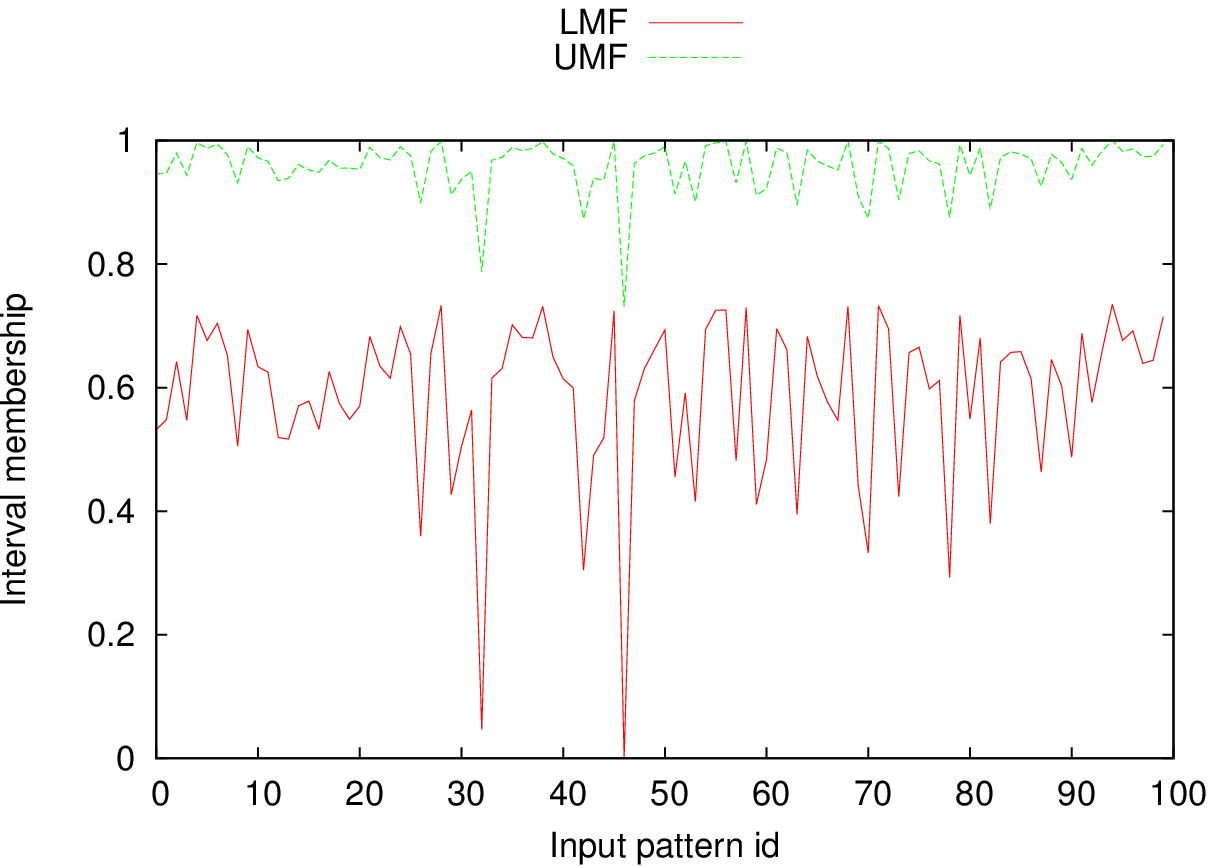}
\label{fig:2dvec_it2}}

\caption{Example of the IT2FS generation over 100 patterns extracted from a 2-dimensional Gaussian distribution.}
\label{fig:2dvec}
\end{figure*}

\subsection{Preservation of Input--Output Uncertainty}
\label{sec:in-out_pres}

Here we calculate, in terms of GE, the quality of the granulation performed by the proposed T2-MinSOD.
We perform the test by generating the dataset $\mathcal{S}$ using a unidimensional Gaussian source (see Eq. \ref{eq:Gaussian_source_entropy} for the entropy expression) and a unidimensional exponential source.
The (Shannon) entropy of the exponential distribution is given by
\begin{equation}
\label{eq:exp_source}
\hat{H}(\mathcal{S}) = 1 - \ln(\lambda),
\end{equation}
where $\lambda>0$ is the scale parameter. From Eq. \ref{eq:exp_source}, it is clear that the entropy decreases as $\lambda$ increases.
For closed-form entropy formulas of other well-known distributions we refer the reader to \cite{Zografos200571}.

First, we performed a batch of 10 experiments by changing the variance, $\sigma$, of the Gaussian source in an suitable range.
Notably, $\sigma$ is progressively selected in $[0.1, 0.55]$, with an increment step of 0.05.
Fig. \ref{fig:bestfit_Gauss} depicts the results of the linear best-fitting (in the figure denoted as a function $f(\cdot)$) among the obtained 10 pairs of source and output entropy values.
The high coefficient of determination ($R^2\simeq 0.98$) denotes a very good relation among the input--output uncertainty; in this case, a linear model is sufficient to map the input--output uncertainty.
The optimization (\ref{eq:ge_epsilon}) to search for the optimal best-fit (dependent only on the width of the Gaussian implementing Eq. \ref{eq:umf}) is performed with a linear search on the $[0.01, 5]$ range with a step-size of 0.01.
Testing of the optimal best-fitting is performed on a new dataset instance generated with $\sigma=0.325$; the test is repeated 10 times by using different random initializations.
We found that $\epsilon \simeq 0.02 \pm 0.005$, which can be considered as a good result, demonstrating thus that the T2-MinSOD is capable of preserving the input uncertainty in the output IG model with a reasonable GE.

We repeated the experiment with an exponential source (\ref{eq:exp_source}), by varying $\lambda$ in $[1.5, 2.4]$ with an increment step of 0.1. The linear best-fit is once again sufficient ($R^2\simeq0.88$) to model the input--output relation among the entropic characterizations of the uncertainty -- see Fig. \ref{fig:bestfit_Exp}.
\begin{figure*}[ht!]
\centering

\subfigure[Gaussian source.]{
 \includegraphics[viewport=0 0 342 243,keepaspectratio,scale=0.55]{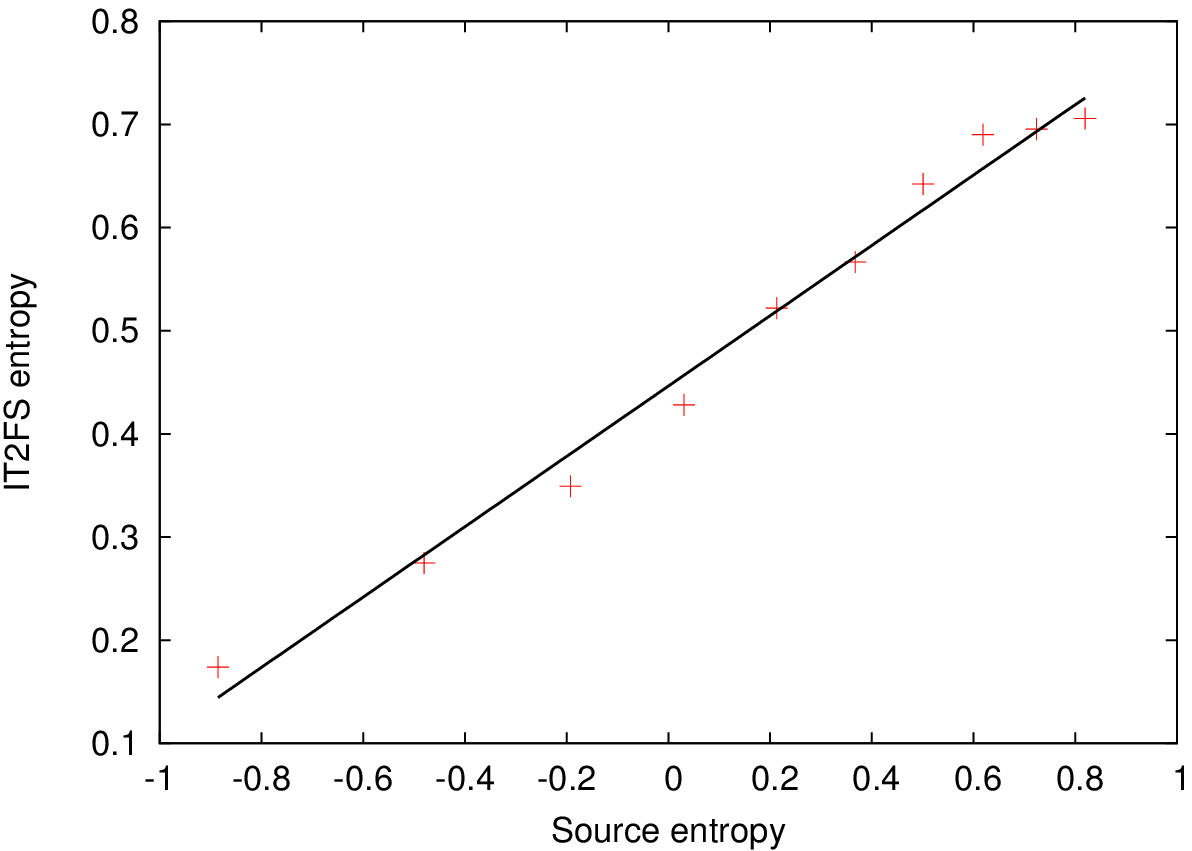}
\label{fig:bestfit_Gauss}}
~
\subfigure[Exponential source.]{
 \includegraphics[viewport=0 0 347 243,scale=0.55,keepaspectratio=true]{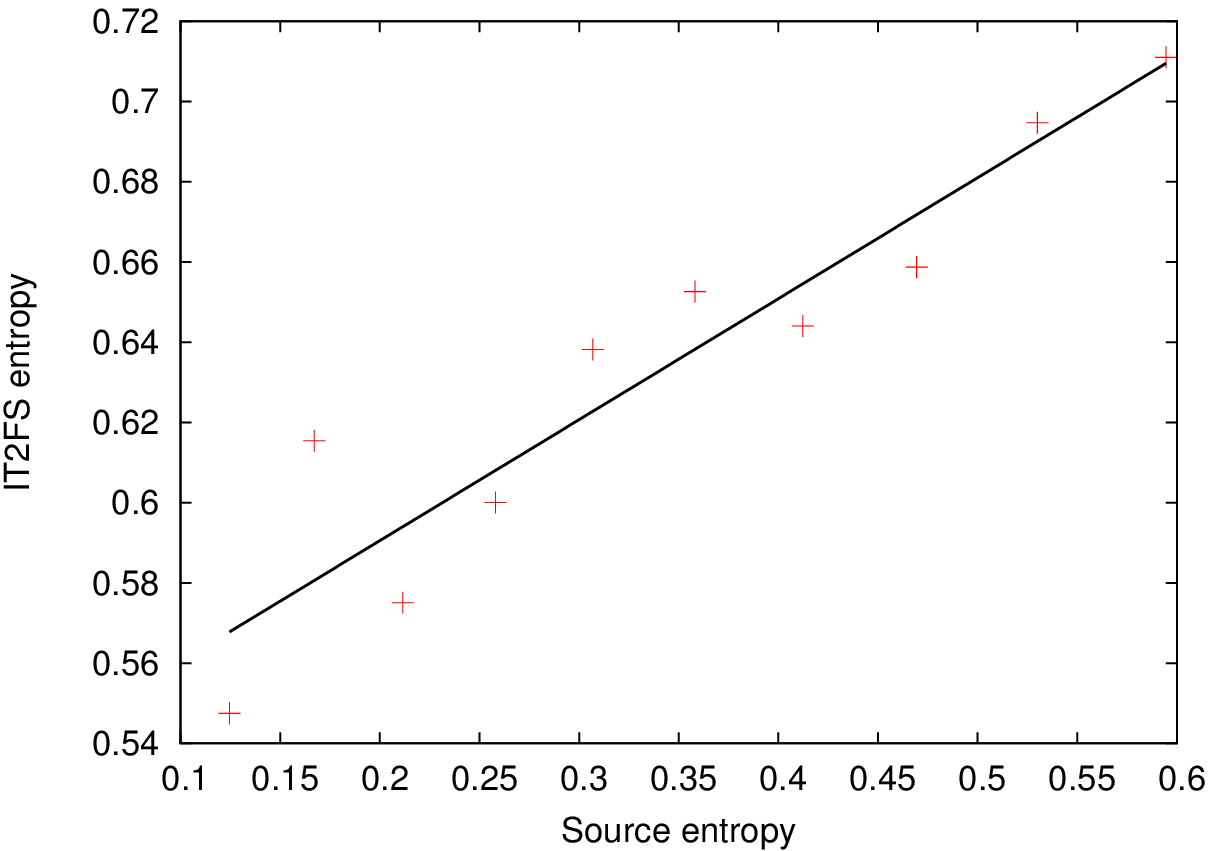}
\label{fig:bestfit_Exp}}

\caption{Best fittings calculated by considering Gaussian (\ref{fig:bestfit_Gauss}) and exponential (\ref{fig:bestfit_Exp}) sources.}
\label{fig:best_fit}
\end{figure*}

\subsection{Tests on Labeled Graphs}
\label{sec:exp_lg}

We demonstrate the modeling capability of the T2-MinSOD when operating in the labeled graphs domain.
We consider both synthetic labeled graphs \cite{seriation+gradis_lncs_2012,odse,gralg_2012} and the letter dataset of the IAM repository \cite{riesen+bunke2008}; those datasets are originally conceived for benchmarking graph classifiers.
In the first case, we consider four out of the 15 original datasets -- datasets contain graphs constructed as Markov chains of decreasing similarity, i.e., related classification problems are intended with decreasing difficulty.
In the latter case we consider the letter dataset with two level of distortions: low (Letter-L) and high (Letter-H) -- this dataset contains graphs representing digitalized letters drawn over the 2D plane.
The dissimilarity measure (\ref{eq:minsod}) is implemented as the graph coverage graph matching algorithm \cite{livi2012_pgm}.

Fig. \ref{fig:graphs} shows the interval widths calculated for the four synthetic datasets of graphs (denoted as DS-G-2, DS-G-6, DS-G-10, and DS-G-14, where DS-G-2 induces a harder classification problems than DS-G-6 and so on).
As expected, the entropy (\ref{eq:it2fs_entropy}) calculated from the IT2FS models is in agreement with the nature of the datasets. More difficult (in terms of recognition) datasets have higher entropy than easier datasets; harder datasets are usually characterized by a less regular pattern organization in the input space.
A similar results holds for Letter-L and Letter-H, shown, respectively, in Figs. \ref{fig:letter1} and \ref{fig:letter2}. The IT2FS model of the easier dataset (Letter-L) denotes less entropy.
\begin{figure*}[ht!]
\centering

\subfigure[DS-G-2. IT2FS entropy=0.232693.]{
 \includegraphics[viewport=0 0 348 247,scale=0.5,keepaspectratio=true]{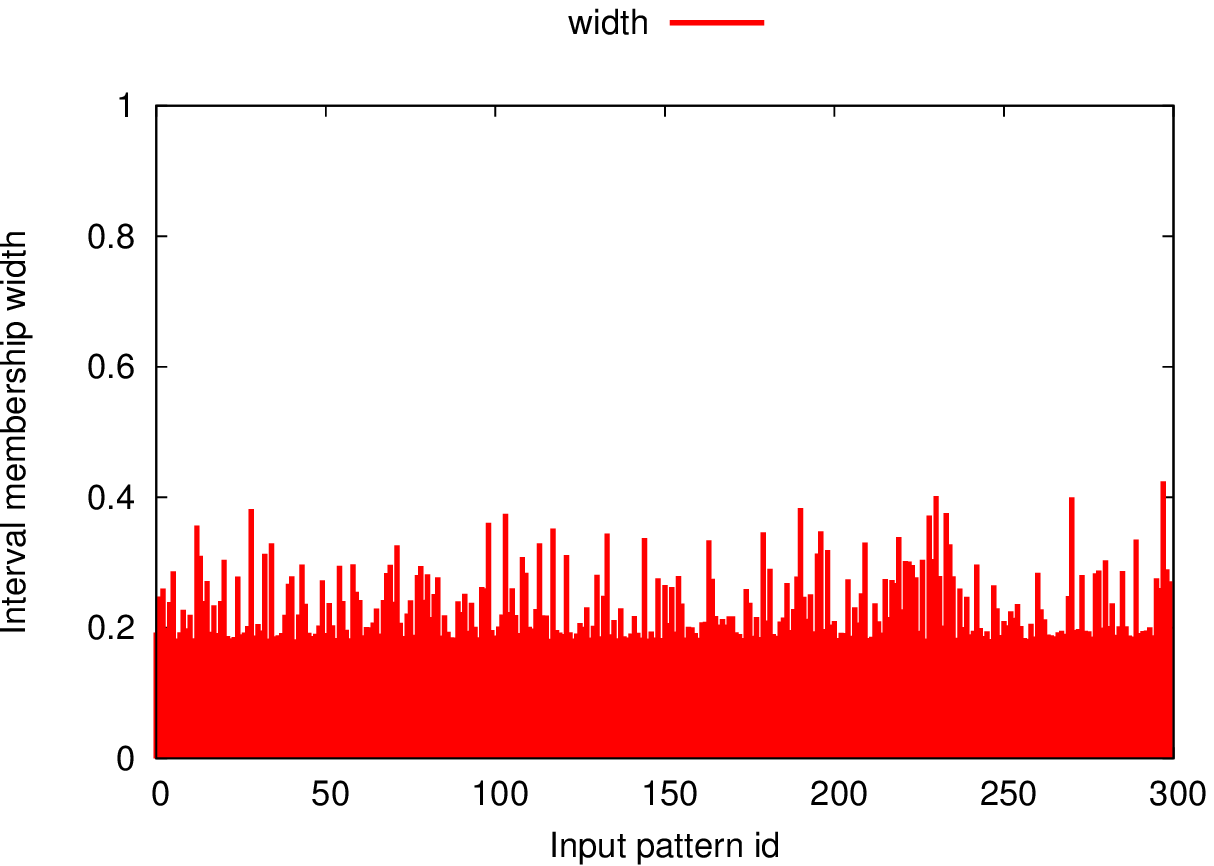}
\label{fig:ds-g-2}}
~
\subfigure[DS-G-6. IT2FS entropy=0.21941.]{
 \includegraphics[viewport=0 0 348 247,scale=0.5,keepaspectratio=true]{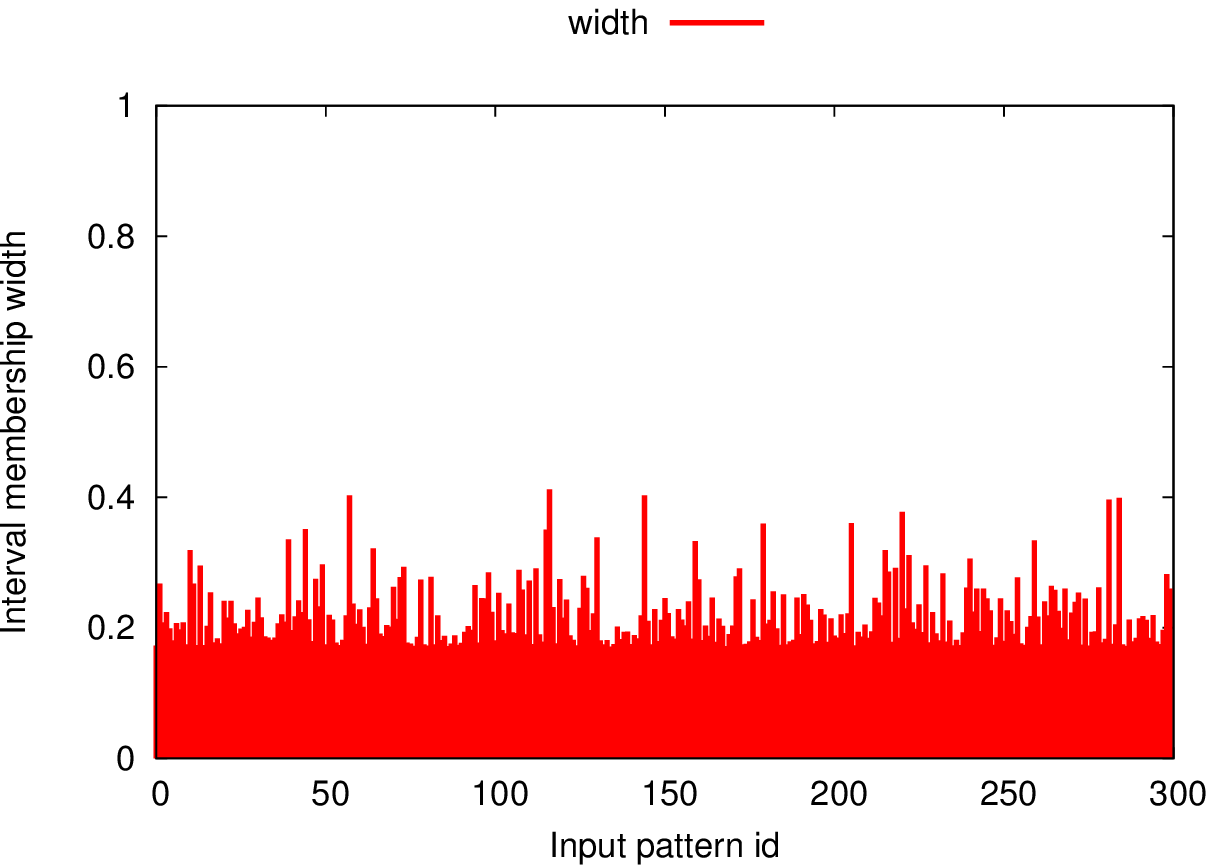}
\label{fig:ds-g-6}}

\subfigure[DS-G-10. IT2FS entropy=0.203789.]{
 \includegraphics[viewport=0 0 348 247,scale=0.5,keepaspectratio=true]{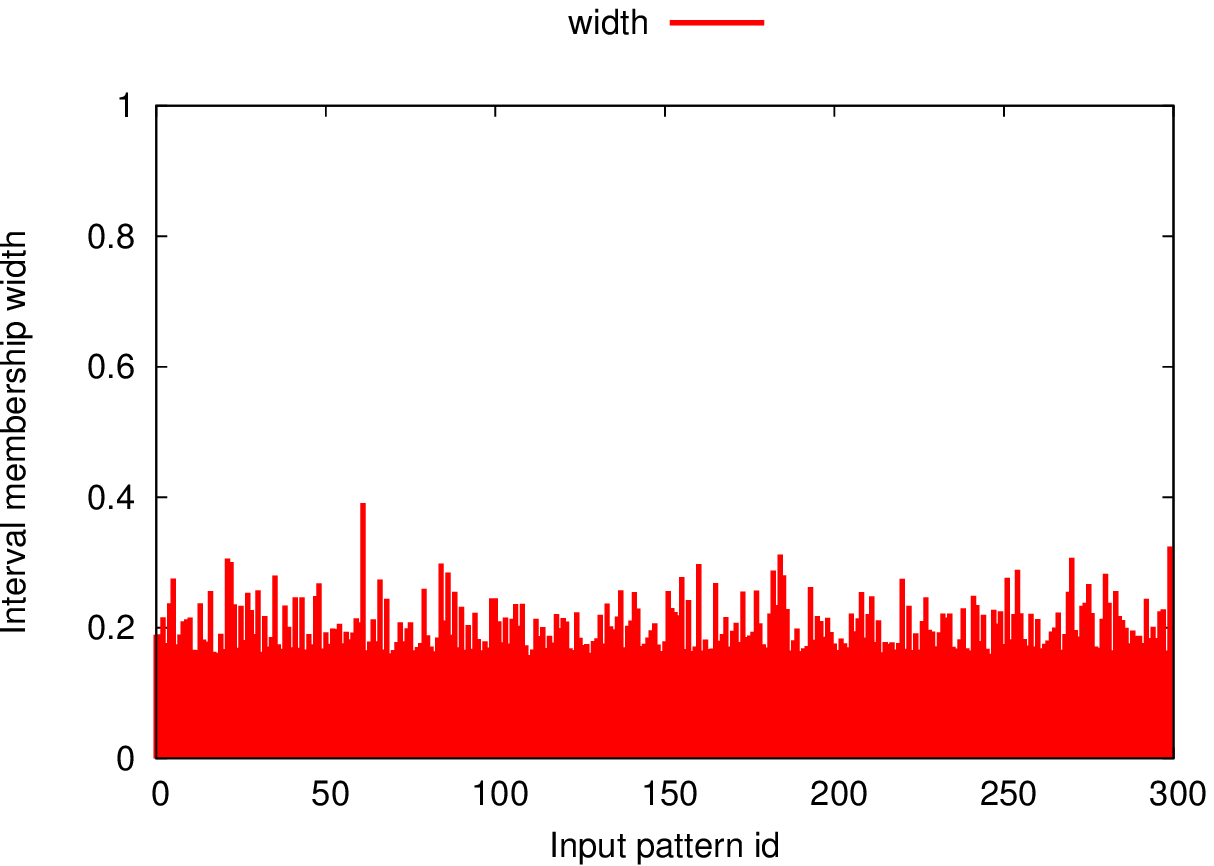}
\label{fig:ds-g-10}}
~
\subfigure[DS-G-14. IT2FS entropy=0.187489.]{
 \includegraphics[viewport=0 0 348 247,scale=0.5,keepaspectratio=true]{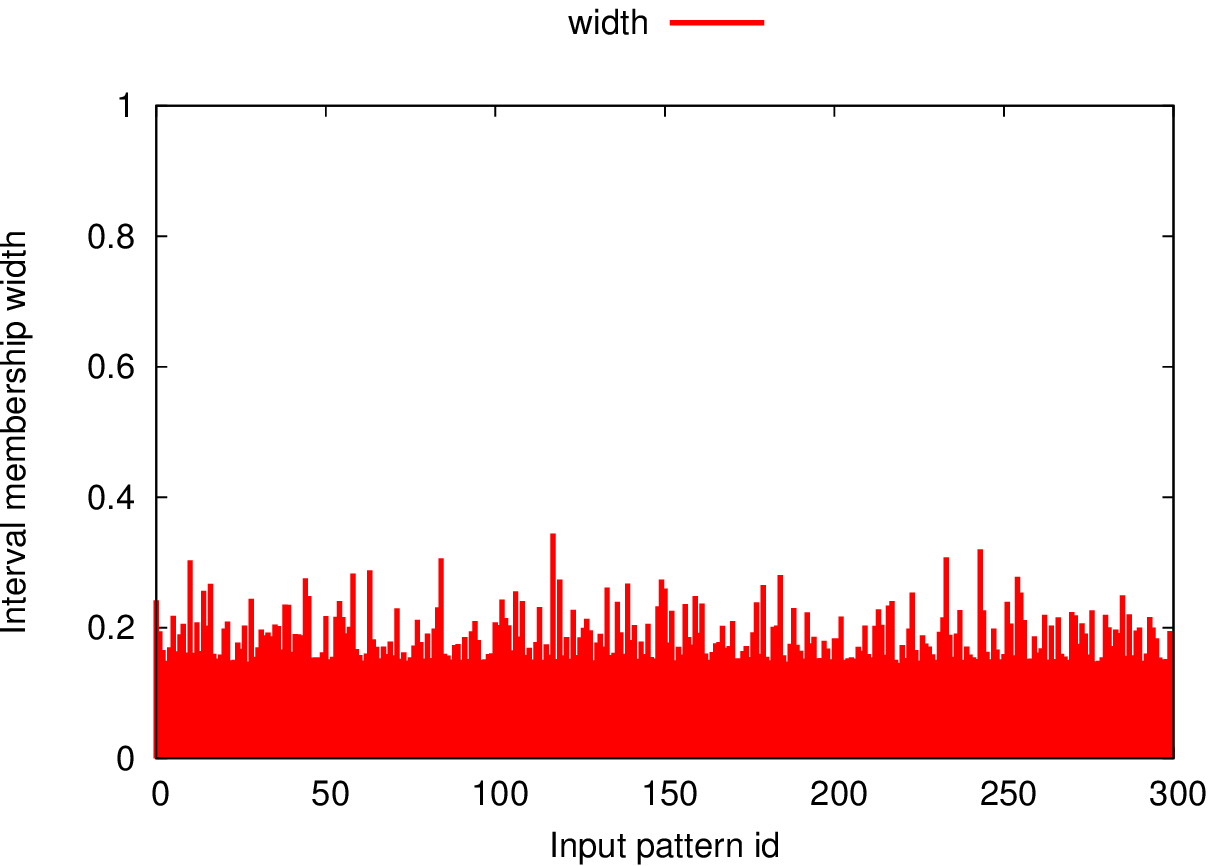}
\label{fig:ds-g-14}}

\caption{Widths and related entropy of the generated interval membership functions (synthetic graphs).}
\label{fig:graphs}
\end{figure*}
\begin{figure*}[ht!]
\centering

\subfigure[Letter-L. IT2FS entropy=0.553205.]{
 \includegraphics[viewport=0 0 339 247,scale=0.5,keepaspectratio=true]{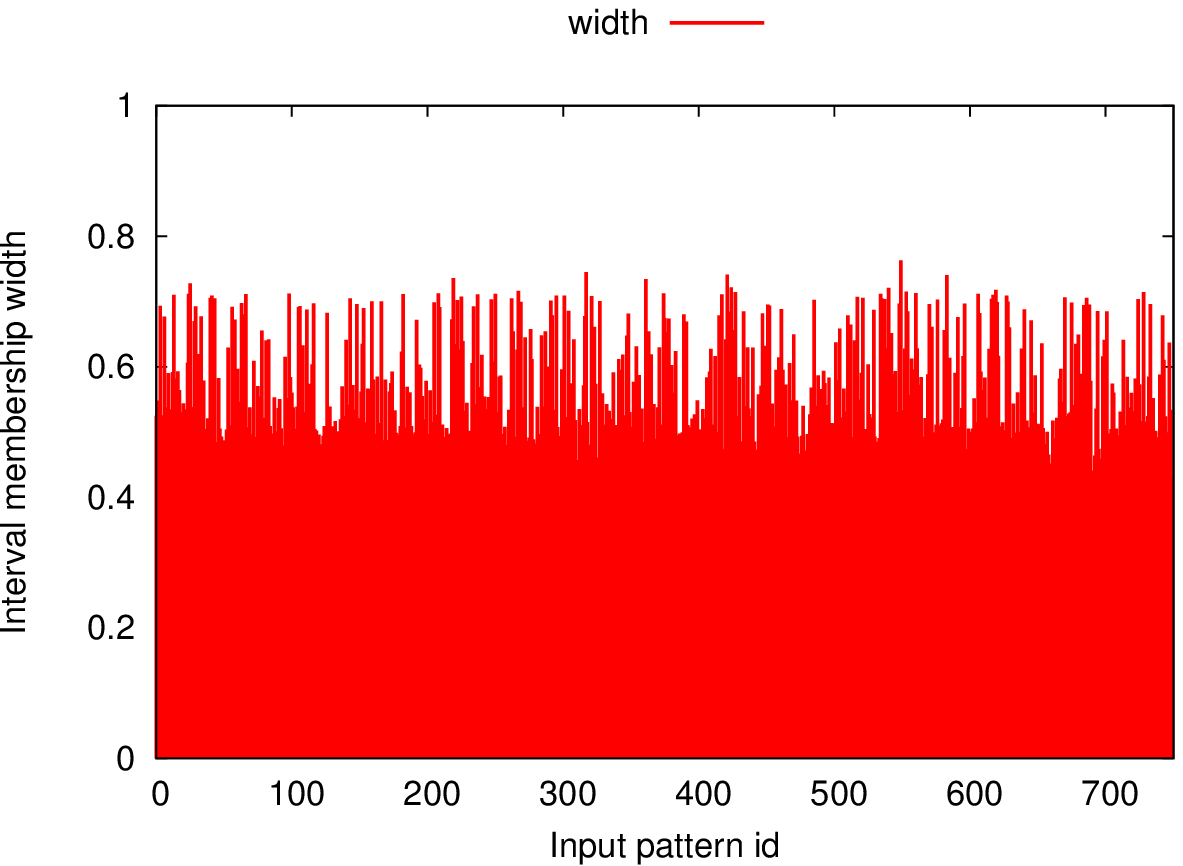}
\label{fig:letter1}}
~
\subfigure[Letter-H. IT2FS entropy=0.628804.]{
 \includegraphics[viewport=0 0 339 247,scale=0.5,keepaspectratio=true]{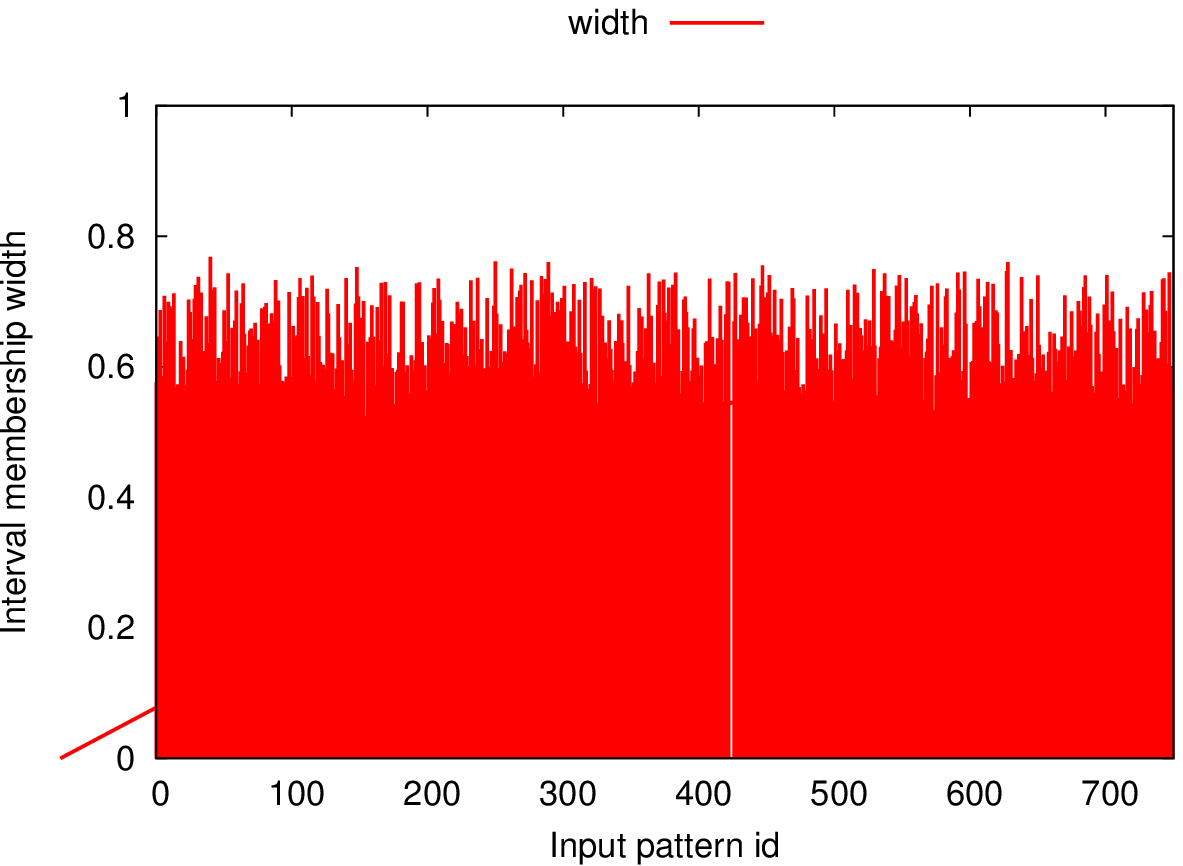}
\label{fig:letter2}}

\caption{Widths and related entropy of the generated interval membership functions (IAM letter datasets).}
\label{fig:graphs_iam}
\end{figure*}

\subsection{T2-MinSOD in Data Clustering}
\label{sec:clustering}

In this section we use the T2-MinSOD to model clusters of data generated with the well-known \textit{k}-means algorithm \cite{cagata,Jain:2010:DCY:1755267.1755654}.
We process the Iris dataset taken from the UCI repository \cite{Bache+Lichman:2013}.
The Iris dataset contains 150 patterns equally distributed in three classes, named ``Iris-setosa'', ``Iris-versicolor'', and ``Iris-virginica''.
From the analysis of first two components of the PCA shown in Fig. \ref{fig:pca_iris}, it is possible to understand that patterns of the ``Iris-setosa'' class (in red) are well-separated from the others, while those of the other two classes show little overlap (according to the PCA space).
This fact suggests us that the uncertainty of a suitable IG modeling patterns belonging to the ``Iris-setosa'' class should be lower than those calculated from IGs describing the other two classes.

To test this hypothesis, we executed the \textit{k}-means algorithm directly in the input space (no pre-processing of data) setting $k=3$, generating thus three IGs modeled by means of the T2-MinSOD.
Fig. \ref{fig:kmeans} shows the widths of generated interval memberships. It is important to note that the first cluster (\ref{fig:uci_iris_c0}) consists of 50 patterns, all belonging to the ``Iris-setosa'' class. The entropy of the related IT2FS model is $\simeq 0.67$.
The other two clusters (shown in Figs. \ref{fig:uci_iris_c1} and \ref{fig:uci_iris_c2}) contain an unbalanced number of patterns: respectively 38 and 62. As fuzzy entropy calculations show, the uncertainty of those two IT2FS models is greater than the one of the first cluster, which agrees with the (visual) information provided by the PCA.
This result suggests that the IT2FS models generated by means of T2-MinSOD convey also useful and reliable higher-level information, to be exploited for interpretability purposes.
\begin{figure}[ht!]
 \centering
 \includegraphics[viewport=0 0 343 241,scale=0.6,keepaspectratio=true]{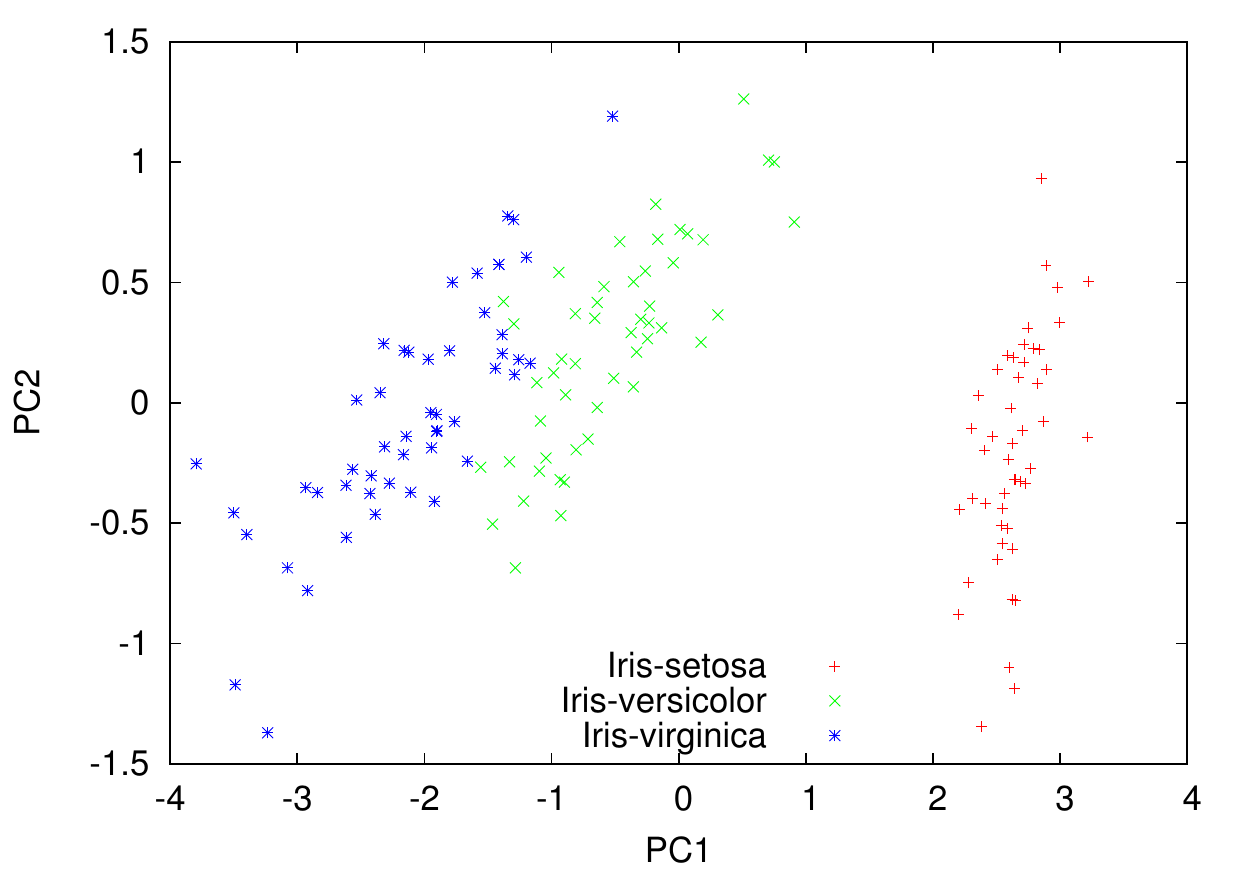}
 \caption{First two components of the PCA of the Iris dataset.}
 \label{fig:pca_iris}
\end{figure}
\begin{figure*}[ht!]
\centering

\subfigure[First cluster. IT2FS entropy=0.670161.]{
 \includegraphics[viewport=0 0 345 247,scale=0.45,keepaspectratio=true]{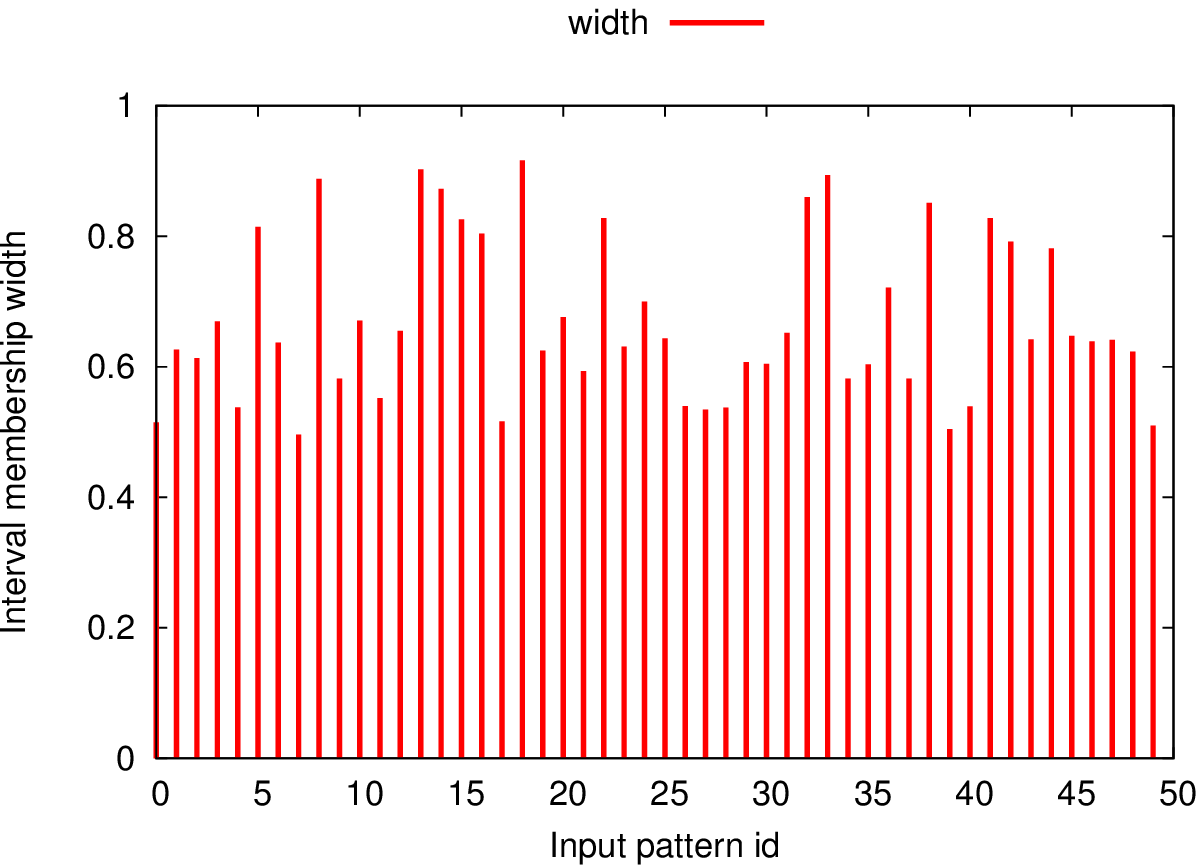}
\label{fig:uci_iris_c0}}
~
\subfigure[Second cluster. IT2FS entropy=0.794648.]{
 \includegraphics[viewport=0 0 345 247,scale=0.45,keepaspectratio=true]{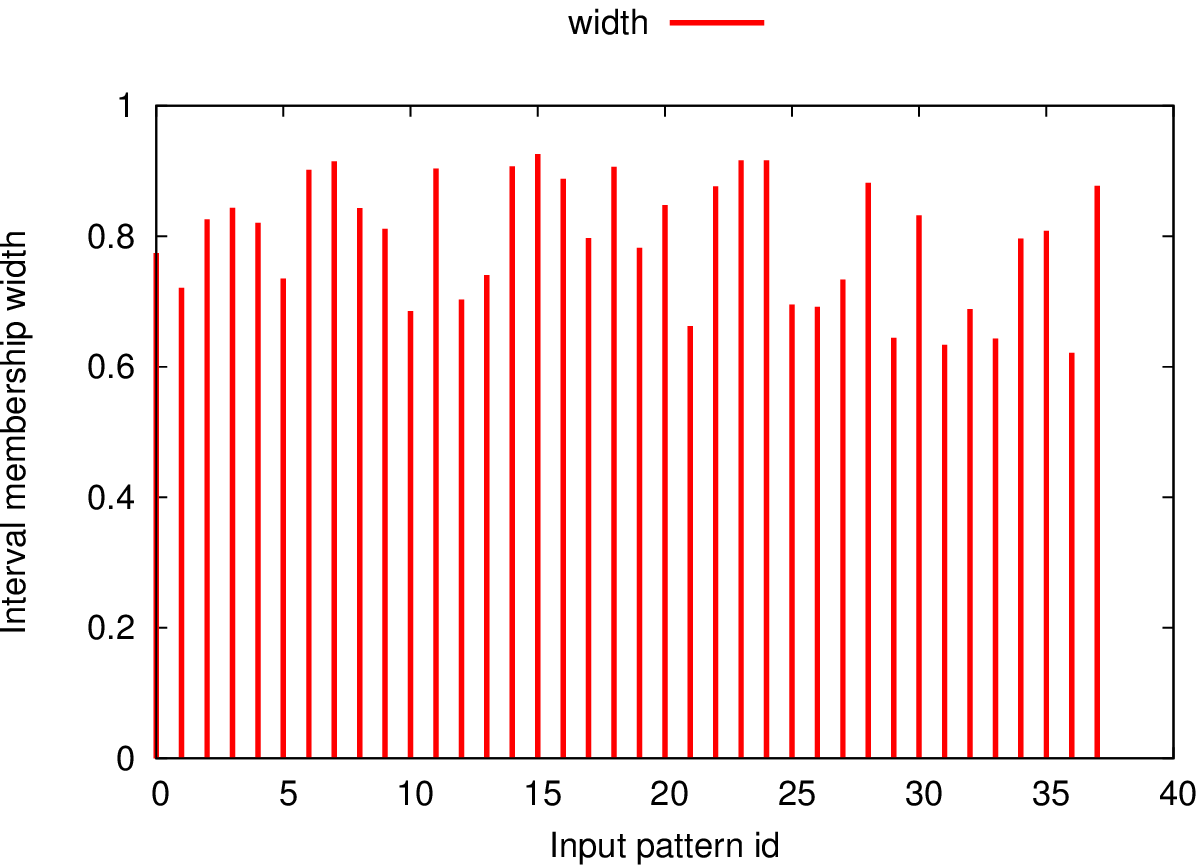}
\label{fig:uci_iris_c1}}
~
\subfigure[Third cluster. IT2FS entropy=0.885963.]{
 \includegraphics[viewport=0 0 345 247,scale=0.45,keepaspectratio=true]{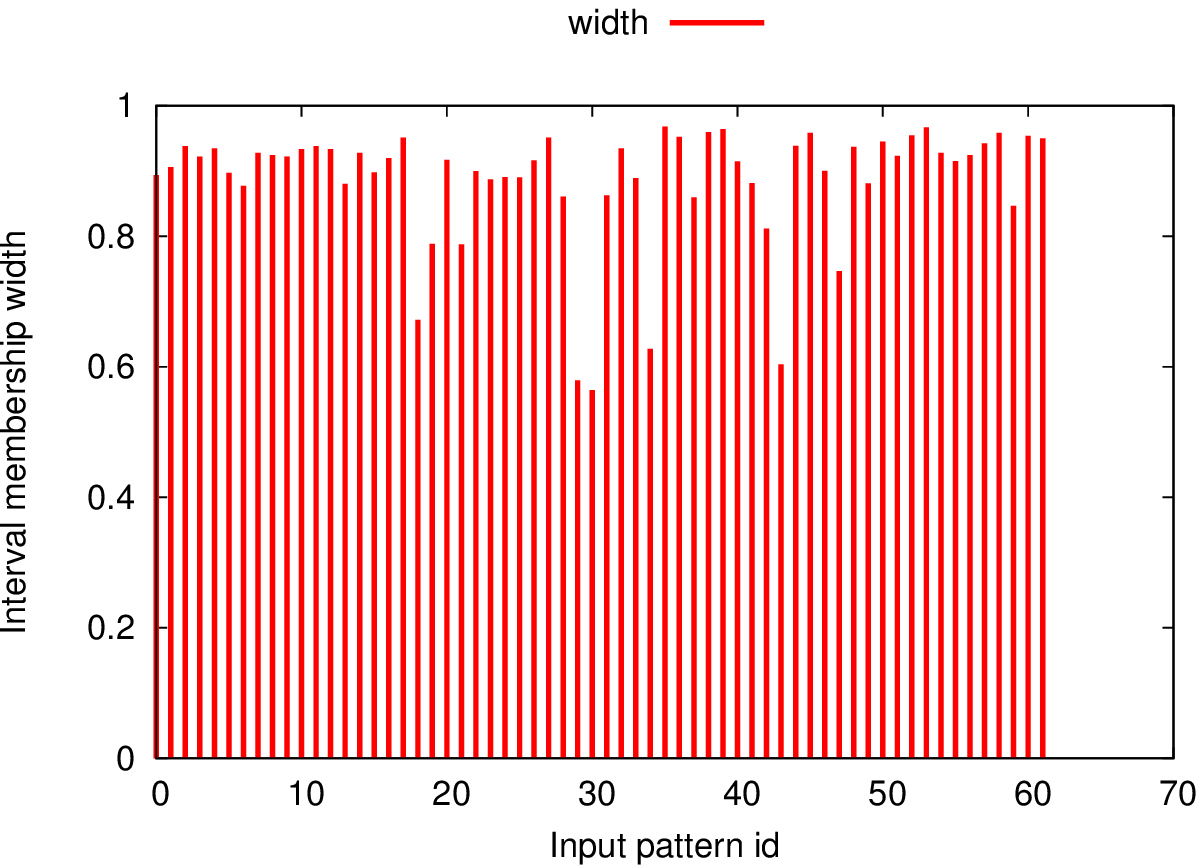}
\label{fig:uci_iris_c2}}

\caption{Widths and related entropy of the generated interval membership functions for the clusters calculated by $k$-means over the Iris dataset.}
\label{fig:kmeans}
\end{figure*}

\section{Conclusions and Future Directions}
\label{sec:conclusions}

The process of data granulation can be abstracted as a mapping among some input domain and a suitable domain of information granules.
In this paper, we have presented a conceptual framework to help designing and evaluating data granulation procedures.
The framework, called PULP, is based on the principles of uncertainty introduced by Klir.
The main idea is to consider the uncertainty of the input data as an invariant property, to be preserved as much as possible in the model of the output IG. The difference among the input and output uncertainty has been defined as the granulation error. Such a quantity has been used to (i) objectively judge over the quality of the granulation and (ii) to provide a common groundwork to compare different granulation procedures operating over the same data.

To put this idea in practice, we introduced a data granulation technique based on the MinSOD.
The procedure, called T2-MinSOD, is able to generate an interval-valued membership function by relying on the information of the input data dissimilarity values only.
We analyzed this procedure by considering different input data types and experimental settings.
Results show that T2-MinSOD is interpretable and it is able to preserve the uncertainty of the input with reasonable granulation errors.

Performing data granulation by considering the uncertainty as an invariant property to be preserved during the granulation process allows to apply this conceptual framework regardless 1) of the specific input data representation formalism and 2) the mathematical setting used to define information granules.

Future works include the theoretical consolidation of PULP.
For instance, it may be interesting to study if the mapping $\phi(\cdot)$ is bijective. This formal property may suggest important facts in terms of IG interpretability. In addition, we will study the so-called ``denagranulation'' by exploiting the invertibility of $\phi(\cdot)$.
Finally, we will use PULP for the purpose of benchmarking different data granulation procedures operating over the same data.

\appendix
\section{Brief Review of Type-2 Fuzzy Sets}
\label{sec:t2fs}

A T2FS $\widetilde{\mathcal{A}}$ defined on the universe of discourse $\mathcal{X}$ is represented as
\begin{align}
\label{eq:t2fs_definition}
\widetilde{\mathcal{A}} =& \{ (x, \mu_{\widetilde{\mathcal{A}}}(x) )\ |\ x\in\mathcal{X}, \\
\nonumber &\mu_{\widetilde{\mathcal{A}}}(x)=\{ (u, g_{x}(u))\ |\ u\in \mathcal{J}_{x}\subseteq[0, 1], g_{x}(u)\in[0, 1] \} \}.
\end{align}

We refer to $\mu_{\widetilde{\mathcal{A}}}(x)$ as the fuzzy membership value of $x$. Moreover, in Eq. \ref{eq:t2fs_definition} $\mathcal{J}_x$ represents the primary membership values of $x$, and $g_{x}(u)$ is named secondary grade \cite{mendel_t2simpe_2002,4565674}.

A T2FS in which $g_{x}(u)=1$ holds $\forall u\in \mathcal{J}_{x}$, reduces to the so-called interval type-2 fuzzy set (IT2FS) \cite{it2fs_matching}.
Please note that an IT2FS is a more general case of what is known in the literature as interval-valued fuzzy set \cite{Vlachos20071384,zeng2006relationship}, where $\mathcal{J}_{x}$ in (\ref{eq:t2fs_definition}) is constrained to be a subinterval of $[0, 1]$.
In this paper, $\mu_{\widetilde{\mathcal{A}}}(\cdot)$ is referred to as interval membership function, since $\mathcal{J}_{x}$ is always a subinterval of $[0, 1]$.
An IT2FS is fully characterized by the so-called Footprint Of Uncertainty (FOU), which is defined as:
\begin{equation}
\label{eq:fou}
\textrm{FOU}(\widetilde{\mathcal{A}}) = \displaystyle\bigcup_{x\in\mathcal{X}} \left[ \underline{\mu}_{\widetilde{\mathcal{A}}}(x), \overline{\mu}_{\widetilde{\mathcal{A}}}(x) \right].
\end{equation}

Note that in (\ref{eq:fou}) $\overline{\mu}_{\widetilde{\mathcal{A}}}(x)$ and $\underline{\mu}_{\widetilde{\mathcal{A}}}(x)$ denote the upper and lower endpoints of $\mathcal{J}_{x}$, respectively.
FOU (\ref{eq:fou}) can be characterized by two T1FSs, which are respectively called Upper Membership Function (UMF) and Lower membership Function (LMF),
\begin{align}
\textrm{UMF}_{\widetilde{\mathcal{A}}} &= \overline{\textrm{FOU}}(\widetilde{\mathcal{A}}) = \left\{ (x, \overline{\mu}_{\widetilde{\mathcal{A}}}(x))\ |\ x\in\mathcal{X} \right\}, \\
\textrm{LMF}_{\widetilde{\mathcal{A}}} &= \underline{\textrm{FOU}}(\widetilde{\mathcal{A}}) = \left\{ (x, \underline{\mu}_{\widetilde{\mathcal{A}}}(x))\ |\ x\in\mathcal{X} \right\}.
\end{align}

\section{The MinSOD Representative}
\label{sec:minsod}

Let $\mathcal{S}\subset\mathcal{X}, n=|\mathcal{S}|$, be a finite input set, and let $d: \mathcal{X}\times\mathcal{X}\rightarrow\mathbb{R}^+$ be a suitable dissimilarity measure \cite{gm_survey,odse}.
The MinSOD \cite{delvescovo+livi+rizzi+frattalemascioli2011} representative element $\nu\in\mathcal{S}$ is the element of $\mathcal{S}$ that minimizes the sum of distances:
\begin{equation}
\label{eq:minsod}
\nu = \argmin_{x_i\in\mathcal{S}} \sum_{x_j\in\mathcal{S}} d(x_i, x_j).
\end{equation}

The prototype $\nu$ computed according to Eq. \ref{eq:minsod} can be though as an ``approximation'' of the centroid in vector spaces. However, being based on a dissimilarity measure, it can be used to model datasets proper of non-geometric input domains, such as those of graphs and sequences \cite{odse}.

In addition to the representative element, i.e., $\nu$, several indicators can be defined. For instance, measures of compactness and size can be easily conceived; the compactness could be conceived as a statistics (e.g., average, standard deviation, etc.) of the dissimilarity values among (a subset of) the elements in $\mathcal{S}$ and $\nu$.

The following claim conveys a useful result from the interpretability viewpoint of the MinSOD.
\begin{thMinsod}
\label{th:Minsod}
The MinSOD element $\nu\in\mathcal{S}=\{x_1, x_2, ..., x_n\}\subset\mathbb{R}$, computed as shown in Eq. \ref{eq:minsod} by setting $d(x_i, x_j)=|x_i-x_j|$, corresponds to the median element of $\mathcal{S}$.
\end{thMinsod}
\begin{proof}
Let $x_i$ be the median of $\mathcal{S}$.
We prove the claim by contradiction, that is, by assuming $\nu=x_k$, with $x_k\neq x_i$.
Let us shorten the dissimilarity as $d(x_i, x_j)=d_{ij}$.
Let us assume that the elements of $\mathcal{S}$ are ordered in ascending order; without loss of generality, we assume also that $x_k$ is located at the right-hand side of $x_i$. This is possible since $x_i$ is the median element.
Since $x_k$ is the MinSOD element, the following inequality must hold:
\begin{equation}
\label{eq:ineq1}
\displaystyle\sum_{j=1}^{n} d_{kj} \leq \displaystyle\sum_{j=1}^{n} d_{ij}.
\end{equation}

By using the fact that
\begin{align}
\displaystyle\sum_{j=1}^{n} d_{ij} = \displaystyle\sum_{j=1}^{n} d_{kj} - \sum_{h=i+1}^{k} d_{ih},
\end{align}
we can rewrite the inequality (\ref{eq:ineq1}) as
\begin{align}
\displaystyle\sum_{j=1}^{n} d_{kj} \leq \displaystyle\sum_{j=1}^{n} d_{kj} - \sum_{h=i+1}^{k} d_{ih}.
\end{align}

By simple manipulations we obtain:
\begin{align}
\displaystyle\sum_{j=1}^{n} d_{kj} - \displaystyle\sum_{j=1}^{n} d_{kj} + \sum_{h=i+1}^{k} d_{ih} \leq 0 \Rightarrow \sum_{h=i+1}^{k} d_{ih} \leq 0,
\end{align}
which is impossible, since at least $d_{ik}>0$ must hold.
\end{proof}

\bibliographystyle{abbrvnat}
\bibliography{/home/lorenzo/University/Research/Publications/Bibliography.bib}
\end{document}